\documentclass[a4paper,12pt]{article}

\usepackage[utf8]{inputenc}
\usepackage[unicode=true,hypertexnames=false]{hyperref}
\usepackage{amsmath,amssymb,amsthm,amsfonts}
\usepackage{mathtools} 
\usepackage[linesnumbered,ruled]{algorithm2e}
\usepackage{xspace}
\usepackage[dvipsnames]{xcolor}
\usepackage{tikz}

\newcommand{\oea}{\mbox{$(1 + 1)$~EA}\xspace}

\newcommand{\oplea}{\mbox{$(1+\lambda)$~EA}\xspace}
\newcommand{\mpoea}{\mbox{$(\mu+1)$~EA}\xspace}
\newcommand{\mplea}{\mbox{$(\mu+\lambda)$~EA}\xspace}

\newcommand{\ollga}{${(1 + (\lambda , \lambda))}$~GA\xspace}

\newcommand{\onemax}{\textsc{OneMax}\xspace}
\newcommand{\leadingones}{\textsc{LeadingOnes}\xspace}
\newcommand{\maxsat}{\textsc{MAX-3SAT}\xspace}
\newcommand{\om}{\textsc{OM}\xspace}
\newcommand{\jump}{\textsc{Jump}\xspace}
\newcommand{\N}{{\mathbb N}}
\newcommand{\R}{{\mathbb R}}

\newcommand{\revise}[1]{#1}

\DeclareMathOperator{\Bin}{Bin}
\DeclareMathOperator*{\argmax}{arg\,max}

\newtheorem{theorem}{Theorem}
\newtheorem{lemma}[theorem]{Lemma}
\newtheorem{corollary}[theorem]{Corollary}

\newcommand{\eps}{\varepsilon} 

\makeatletter 
\renewcommand*{\@fnsymbol}[1]{\ensuremath{\ifcase#1\or\dagger\or*\else\@arabic{#1}\fi}}
\makeatother

\let\originalleft\left
\let\originalright\right
\renewcommand{\left}{\mathopen{}\mathclose\bgroup\originalleft}
\renewcommand{\right}{\aftergroup\egroup\originalright}

\usepackage{hyperref}

\title{A Rigorous Runtime Analysis of the $(1 + (\lambda, \lambda))$ GA on Jump Functions\thanks{Extended version of the paper~\cite{AntipovDK20} in the proceedings of GECCO \revise{2020}. This version contains all proofs and other details that had to be omitted in the conference version for reasons of space. Also, we have added a new section which proves the lower bounds.}}

\author{Denis Antipov\thanks{Corresponding author} \\
		ITMO University \\
  		St. Petersburg, Russia \\
  		and \\
		Laboratoire d'Informatique (LIX), \\
		CNRS, \'Ecole Polytechnique, \\ 
		Institut Polytechnique de Paris \\
    Palaiseau, France \\
    antipovden@yandex.ru 
		\and
		Benjamin Doerr \\
		Laboratoire d'Informatique (LIX), \\
		CNRS, \'Ecole Polytechnique, \\ 
		Institut Polytechnique de Paris \\
    Palaiseau, France \\
    lastname@lix.polytechnique.de 
		\and
		Vitalii Karavaev\\
		ITMO University \\
    St. Petersburg, Russia \\
    fkve97@gmail.com 
}

\begin{document}

\maketitle

{\sloppy

\begin{abstract}
    The $(1 + (\lambda,\lambda))$ genetic algorithm is a younger evolutionary algorithm trying to profit also from inferior solutions. Rigorous runtime analyses on unimodal fitness functions showed that it can indeed be faster than classical evolutionary algorithms, though on these simple problems the gains were only moderate.

    In this work, we conduct the first runtime analysis of this algorithm on a multimodal problem class, the jump functions benchmark. We show that with the right parameters, the \ollga optimizes any jump function with jump size $2 \le k \le n/4$ in expected time $O(n^{(k+1)/2} e^{O(k)} k^{-k/2})$, which significantly and already for constant~$k$ outperforms standard mutation-based algorithms with their $\Theta(n^k)$ runtime and standard crossover-based algorithms with their $\revise{\tilde{O}}(n^{k-1})$ runtime guarantee.
    
    For the isolated problem of leaving the local optimum of jump functions, we determine provably optimal parameters that lead to a runtime of $(n/k)^{k/2} e^{\Theta(k)}$. This suggests some general advice on how to set the parameters of the \ollga, which might ease the further use of this algorithm.
\end{abstract}

\section{Introduction}

The $(1 + (\lambda,\lambda))$ genetic algorithm, \ollga for short, is a relatively new genetic algorithm, first proposed at GECCO 2013~\cite{DoerrDE13}, that tries to increase the rate of exploration by a combination of mutation with a high mutation rate, an intermediate selection, and crossover as mechanism to repair the possible negative effects of the aggressive mutation. For this algorithm, moderate runtime advantages over classic algorithms have been proven for unimodal~\cite{DoerrDE15,DoerrD18} or close-to-unimodal~\cite{BuzdalovD17} problems; also some positive experimental results exist~\cite{GoldmanP14,MironovichB17}.

In this work, we conduct the first mathematical runtime analysis for the \ollga optimizing a multimodal optimization problem, namely the classic jump functions benchmark. We observe that the combination of aggressive mutation with crossover as repair mechanism works even better here: The \ollga can optimize jump functions with gap size $k \le n/4$ in expected time at most \[n^{(k+1)/2} e^{O(k)} k^{-k/2},\] which is almost the square root of the $\Omega(n^{k})$ runtime many classic mutation-based algorithms have. To obtain this performance, however, the parameters of the algorithm have to be set differently from what previous works recommend.

\subsection{The \ollga}
\label{sec:intro-ollea}


Noting that many classic evolutionary algorithms do not profit a lot from inferior solution, whereas algorithms witnessing the black-box complexity~\cite{DrosteJW06} (see also~\cite{DoerrD20bookchapter} for a recent survey), massively do, Doerr, Doerr, and Ebel~\cite{DoerrDE13} proposed an algorithm which tries to gain some insight also from solutions inferior to the current-best solution. 

The main working principle of their algorithm, which was called \ollga, is as follows. From a unique parent individual $x$, first $\lambda$ offspring are created using standard bit mutation with a relatively high mutation rate $p$, \revise{but in a way that all offspring have equal Hamming distance to~$x$ (this can be realized, for example, by first sampling a number $\ell$ from a binomial distribution with parameters $n$ and $p$ and then generating all offspring by flipping exactly $\ell$ random bits in $x$)}. When the parent is already close to the optimum, these most likely are all worse than the parent. The hope set into the \ollga is that nevertheless some mutation offspring, besides all destruction from the aggressive mutation, has also made some progress. To distill such progress the \ollga selects a mutation offspring $x'$ with maximal fitness and creates from it, $\lambda$~times independently, an offspring via a biased uniform crossover with the parent $x$. This biased crossover inherits bits from $x'$ only with some small probability~$c$, so that, hopefully, all the destruction caused by the aggressive mutation is repaired. The best of these crossover offspring in an elitist selection competes with $x$ for becoming the parent of the next iteration. The recommendation in previous works was to use a crossover bias of $c = \frac{1}{p n}$. With this parameterization, a single application of mutation and crossover with the parent, without intermediate selection, would create an offspring distributed as if generated via standard bit mutation with mutation rate $\frac 1n$. Note that $\frac 1n$ is \revise{a common} recommendation for the mutation rate in standard bit mutation~\cite{Back93,Muhlenbein92,Witt13}.

Via a rigorous runtime analysis on the \onemax benchmark function it was shown~\cite{DoerrDE15} that the basic idea of the \ollga indeed can work. When the crossover biased is set to $c = \frac{1}{p n}$ as recommended, then the expected runtime (number of fitness evaluations) of the \ollga with any mutation rate $p \ge \frac 2n$ and offspring population size $\lambda \ge 2$ is 
\[
O((\tfrac 1 {p n} + \tfrac 1\lambda) n \log n + (p n + \lambda) n).
\]
Hence any choice of $p  \in \omega(\frac 1n) \cap o(\frac{\log n}{n})$ and $\lambda \in \omega(1) \cap o(\log n)$ yields a runtime asymptotically faster than the runtime $\Omega(n \log n)$ observed by many classic evolutionary algorithms, e.g., by the \oea~\cite{Muhlenbein92}, the \oplea~\cite{JansenJW05}, the \mpoea~\cite{Witt06}, the \mplea~\cite{AntipovD21algo}, and in fact any unary unbiased black-box algorithm~\cite{LehreW12}. The choice $p = \frac{\sqrt{\log n}}{n}$ and $\lambda = \sqrt{\log n}$ minimizes the runtime guarantee above and shows an expected runtime of $O(n \sqrt{\log n})$. With \revise{a fitness-dependent~\cite{DoerrDE15}, self-adjusting~\cite{DoerrD18}, or heavy-tailed} random parameter choice~\cite{AntipovBD20gecco}, the runtime further improves to $O(n)$. Clearly, these are not a drastic improvement over, say, the $O(n \log n)$ runtime of the \oea, but one has to admit that the room for improvement is limited: The unrestricted black-box complexity of the \onemax function class is $\Omega(\frac{n}{\log n})$~\revise{\cite{ErdosR63,DrosteJW06}}, hence no black-box optimizer can optimize all functions isomorphic to \onemax in a time better than $O(\frac{n}{\log n})$. 

A runtime analysis~\cite{BuzdalovD17} of the \ollga on the random satisfiability instances regarded in~\cite{SuttonN14} showed a similar performance as on \onemax. This is caused by the structure of these random instances, which renders them similar to \onemax to the extent that also the \oea has an $O(n \log n)$ performance~\cite{DoerrNS17}. At the same time, these instances do not have the perfect fitness-distance correlation of the \onemax function, and this indeed needed to be taken into account when setting the parameters of the \ollga in~\cite{BuzdalovD17}. A runtime analysis of the \ollga on \leadingones~\cite{AntipovDK19foga} showed that for this problem, the \ollga with any $\lambda \le \frac{n}{2}$ has asymptotically the same runtime of $\Theta(n^2)$ as many other algorithms. 

Empirical studies showed that the \ollga works well (compared to classic EAs) on linear functions and RoyalRoad functions~\cite{DoerrDE15}, on the \maxsat problem~\cite{GoldmanP14}, and on the problem of hard test generation~\cite{MironovichB17}.

\subsection{Multimodal Problems}

Clearly, the usual application of evolutionary algorithms are problems with multimodal landscapes, that is, with non-trivial local optima, and these local optima often present a difficulty for the evolutionary algorithm. In the runtime analysis perspective multimodal problems have displayed very different optimization behaviors. For example, on multimodal landscapes it has been observed that crossover can recombine solutions into significantly better ones~\cite{JansenW02,StorchW04,Sudholt05}, that mutation rates significantly larger than~$\frac 1n$ can be preferable~\cite{DoerrLMN17}, and that probabilistic model-building algorithms such as estimation-of-distribution algorithms and ant-colony optimizers can significantly outperform classic algorithms~\cite{HasenohrlS18,Doerr21cgajump,DoerrK20evocop,BenbakiBD21}.

In this light, and given that all previous runtime analyses for the \ollga consider unimodal or almost unimodal problems, we feel that it is the right time to now investigate how the \ollga optimizes multimodal problems. Being the most studied multimodal benchmark in runtime analysis, we regard jump functions. These have a fitness landscape isomorphic to the one of \onemax except that there is a valley of low fitness around the optimum. Consequently, a typical hillclimber and also most evolutionary algorithms quickly run into the local optimum consisting of all points on the edge of the fitness valley, but then find it hard to cross the fitness valley. 

More precisely, the jump function class comes with a difficulty parameter~$k$, which is the width of the valley of low fitness. The fitness is essentially the fitness of \onemax except for all search points with Hamming distance between one and $k-1$ from the optimum. Consequently, the only way to leave the local optimum to a strictly better search point is to flip exactly the right $k$ bits and go to the optimum.\footnote{This particular structure of the jump benchmark has been criticized and several variants have been proposed~\cite{Jansen15,RajabiW21gecco,BamburyBD21}. \revise{With the overwhelming majority of the runtime analyses on multimodal problems still regarding the classic jump benchmark}, for the sake of comparability we prefer to regard this benchmark as well.}  For this reason, it comes as no surprise that many mutation-based evolutionary algorithms need $\Omega(n^k)$ time to optimize such a jump function~\cite{DrosteJW02,Doerr20gecco}. Using \revise{a higher or a heavy-tailed random} mutation rate~\cite{DoerrLMN17} or stagnation detection mechanisms~\cite{RajabiW20,RajabiW21evocop,RajabiW21gecco}  \revise{the runtime can be reduced, but not below} $\Omega((\frac{n}{k})^k)$. Crossover can be helpful, but the maybe most convincing work~\cite{DangFKKLOSS18} in this direction also only obtains a runtime of $O(n^{k-1} \log n)$ with the standard mutation rate and $O(n^{k-1})$ with a higher mutation rate. With additional techniques, runtimes up to $O(n)$ were obtained~\cite{DangFKKLOSS16,FriedrichKKNNS16,WhitleyVHM18,RoweA19}, but the lower the runtimes become, the more these algorithms appear custom-tailored to jump functions \revise{(see, e.g.,~\cite{Witt21})}. The extreme end is marked by an $O(\frac{n}{\log n})$ time algorithm~\cite{BuzdalovDK16} designed to witness the black-box complexity of jump functions. 

\subsection{Our Results}

Our main result is a runtime analysis of the \ollga on jump functions for all jump sizes $k \in [2..\frac n{4} ]$. Since we could not be sure that the parameter suggestions from previous works are still valid for our problem, we consider arbitrary values for the mutation rate $p$, the crossover bias $c$, and the offspring population size $\lambda$. This turned out to be the right decision as we observed much better runtimes with novel parameter values\footnote{\revise{In~\cite{FajardoS20} it was shown that the \ollga with the standard parameter setting is not effective on \jump even when parameter control mechanisms are applied to $\lambda$. This is another reason to step back from the standard parameters and consider a wider parameters space.}}. We also allowed different offspring population sizes $\lambda_M$ and $\lambda_C$ for the mutation and crossover phase, which however did not lead to stronger runtime guarantees. 

For all $k \in [2..\frac{n}{4}]$ and for arbitrary values of these four parameters (except for the only constraint $p \ge \frac{2k}{n}$), 
we prove that the \ollga when started in the local optimum of $\jump_k$ crosses the fitness valley in expected time (number of fitness evaluations) at most
\[
E\left[T\right] \le \frac{4 (\lambda_M + \lambda_C)}{q_\ell \min\{1, \lambda_M (\frac{p}{2})^k\} \min\{1, \lambda_C c^k \left(1 - c\right)^{2p n - k}\}},
\]
where $q_\ell$ is a constant in $[0.1, 1]$. When ignoring the hidden constants in the $e^{O(k)}$ factor, this bound is optimized for $p = c = \sqrt{\frac kn}$ and $\lambda_M = \lambda_C = n^{k/2} k^{-k/2}$ and then gives a runtime of 
\[
E[T] = n^{k/2} e^{O(k)} k^{-k/2}.
\]
This time bound is asymptotically optimal, that is, no other parameter values can obtain a faster expected runtime (apart from the unspecified $e^{O(k)}$ factor).

When not starting in the local optimum, but with an arbitrary initial solution or the usual random initialization, the \ollga reaches the local optimum in an expected time of $n e^{O(k)}$ iterations, if $p = c = \sqrt{\frac kn}$ and $\lambda_M$ and $\lambda_C$ are at least $\frac{n}{k}$. Therefore, large population sizes are not beneficial in this first easy part of the optimization. With slightly smaller values for the population sizes as above, namely $\lambda_M = \lambda_C = n^{(k - 1)/2} k^{-k/2}$, the expected runtime is
\[
E[T] \le n^{(k + 1)/2} e^{O(k)} k^{-k/2}.
\] 
Similar as in the previous results on \onemax, a speed-up over classic algorithms is observed for larger ranges of parameters, though these are harder to describe in a compact fashion (see Corollary~\ref{thm:parameter-range} for the details). 

The result above shows that the power of the \ollga becomes much more visible for jump functions than for the problems regarded in previous works. Concerning the optimal parameter values, we observe that they differ significantly from those that were optimal in the previous works. In particular, the relation of mutation rate and crossover bias is different. Whereas in previous works $p c n = 1$ was a good choice, we now have $p c n = k$. A moment's thought, however, shows that this is quite natural, or, being more cautious, at least fits to the previous results. We recall that $p c n$ is the expected Hamming distance of the parent from an individual generated from one isolated application of mutation and crossover. The previous works suggested that this number should be one, since one is also the expected distance of an offspring generated the classic way, that is, via standard bit mutation with mutation rate $\frac 1n$. 

Now for the optimization of jump functions, where a non-trivial local optimum has to be left, it makes sense to put more weight on larger moves in the search space. More specifically, the work~\cite{DoerrLMN17} has shown that the optimal mutation rate for the \oea optimizing jump functions is~$\frac kn$. Hence for the classic \oea, the best way of generating offspring is such that they have an expected Hamming distance of $k$ from the parent. Clearly, this remains an intuitive argument, but it shows that also when optimizing multimodal problems, the intuitive approach of previous works, which might help an algorithm designer, gave the right intuition. 


Our recommendation when using the \ollga for multimodal optimization problems would therefore be to choose $p$ and $c$ larger than in previous works, and more specifically, in a way that $p c n$ is equal to an estimate for the number of bits the algorithm typically should flip. Here ``typically'' does not mean that there are actually many moves of this size, but that this is the number of bits the algorithm has to flip most often. For example, when the \oea optimizes a jump function, it will maybe only once move to a search point in distance $k$, however, it will nevertheless need many offspring in distance $k$ until it finds the right move of this distance. 

From our rigorous analysis, we conclude that the \ollga is even better suited for the optimization of multimodal objective functions, and we hope that the just sketched intuitive considerations help algorithm designers to successfully apply this algorithm to their problems.

\textbf{Research conducted after ours:} In~\cite{AntipovD20ppsn}, it was shown that the non-trivial choice of the parameters of the \ollga when optimizing multimodal problems can partially be overcome by using heavy-tailed random parameter values. If we choose $\lambda$ from a power-law distribution with exponent $\beta_\lambda > 2$ and set $p = c = \sqrt{\frac{s}{n}}$, where $s$ follows another power-law distribution with exponent $\beta_s > 1$, then for all $k \ge 3$ the runtime of the \ollga on $\jump_k$ is $(n/k)^{(1 + \eps)k/2} e^{\Theta(k)}$ for any small constant $\eps > 0$, which is only by a $(n/k)^{\eps k/2} n^{-1/2}$ factor larger (and for some $k$ and $\eps$ even smaller) than the upper bound for the optimal static parameters (apart from the unspecified $e^{O(k)}$ factors).

In~\cite{AntipovBD21gecco} it was further shown that if all three parameters of the \ollga are chosen independently, then the runtime stays the same, namely $(n/k)^{(1 + \eps)k/2} e^{\Theta(k)}$. The empirical analysis in~\cite{AntipovBD21gecco} also shows that the \ollga with the heavy-tailed choice of parameters significantly outperforms the \oea on small jump sizes ($k = 3$ and $k = 5$ were considered). 



%

\section{Preliminaries and Notation}

\subsection{Notation}
\label{sec:notation}

By $\N$ we understand the set of positive integers. We write $[a..b]$ to denote an integer interval including its borders and $(a..b)$ to denote an integer interval excluding its borders. For $a, b \in \R$ the notion $[a..b]$ means $[\lceil a \rceil..\lfloor b \rfloor]$. For the real-valued intervals we write $[a, b]$ and $(a, b)$ respectively.
For any probability distribution $\mathcal{L}$ and random variable $X$, we write $X\sim\mathcal{L}$ to indicate that $X$ follows the law $\mathcal{L}$.
We denote the binomial law with parameters $n \in \N$ and $p \in [0,1]$ by $\Bin\left(n, p\right)$.

\subsection{The \ollga}
\label{sec:algorithm}


The main idea of the \ollga discussed in Section~\ref{sec:intro-ollea} is realized as follows.
The \ollga stores a bit string $x$ that is initialized with a random bit string. 
After the initialization it performs iterations which consist of a mutation phase and a crossover phase until some stopping criterion is met.

In the mutation phase the algorithm first chooses the mutation strength $\ell$ from the binomial distribution with parameters $n$ and $p$.
Then it creates $\lambda_M$ mutants $x^{(1)}, \dots, x^{(\lambda_M)}$, each of them is a copy of $x$ with exactly $\ell$ bits flipped. 
The positions of the flipped bits are chosen uniformly at random, independently for each mutant. The goal of this design of the mutation phase is to generate each of $\lambda_M$ offspring via standard bit mutation, but conditional on that all offspring have the same distance to their parent $x$.  
The mutant $x'$ with the best fitness is chosen as a winner of the mutation phase. 

In the crossover phase the algorithm creates $\lambda_C$ offspring $y^{(1)}, \dots,$ $y^{(\lambda_C)}$ by applying a biased crossover to $x$ and $x'$. 
The crossover operator for each position takes a bit value from $x$ with probability $1 - c$ and it takes a bit value from $x'$ with probability $c$ (independently for each position and each offspring).
If the best offspring $y$ is not worse than $x$ then it replaces $x$. 
The pseudocode of the \ollga optimizing a pseudo-Boolean function $f$ is shown in Algorithm~\ref{alg:pseudo}.
\begin{algorithm}[h]
    $x \gets $ random bit string of length $n$\;
    \While{not terminated}
        {
        \textbf{Mutation phase:}\\
        Choose $\ell \sim \Bin\left(n, p\right)$\;
        \For{$i \in [1..\lambda_M]$}
            {$x^{(i)} \gets$ a copy of $x$\;
            Flip $\ell$ bits in $x^{(i)}$ chosen uniformly at random\;
            }
        $x' \gets \argmax_{z \in \{x^{(1)}, \dots, x^{(\lambda)}\}}f(z)$\;
        \textbf{Crossover phase:}\\
        \For{$i \in [1..\lambda_C]$}
            {$y^{(i)} \gets$ a copy of $x$\;
            Flip each bit in $y^{(i)}$ that is different in $x'$ with probability $c$\;
            }
        $y \gets \argmax_{z \in \{y^{(1)}, \dots, y^{(\lambda)}\} }f(z)$\;
        \If{$f(y) \ge f(x)$}
            {
             $x \gets y$\;   
            }
        }
    \caption{The {${(1 + (\lambda , \lambda))}$~GA\xspace} maximizing function $f: \{0,1\}^n \to \R$.}
    \label{alg:pseudo}
\end{algorithm}

We intentionally do not specify a stopping criterion, which is a common practice in theoretical studies. The goal of our analysis is to determine the expected runtime of the \ollga until it finds an optimal solution. By the runtime we understand the number of iterations or fitness evaluations which the algorithm performs. Since each iteration of the algorithm uses exactly $\lambda_M + \lambda_C$ fitness evaluations, the transition between these two measures of runtime is trivial.



\subsection{Jump Functions}
\label{sec:jump}

The class of jump functions is defined through the classic \onemax function, which is defined on the space of bit strings of length $n$ and returns the number of one-bits in its argument. In formal words, 
\[
  \onemax(x) = \om(x) = \sum_{i = 1}^n x_i.  
\]
This function despite its simplicity has given a birth to many fundamental results, e.g.~\cite{Droste02,Droste04,Droste05,JansenJW05,DoerrHK12,Witt13,RoweS14,BadkobehLS14}. In particular, the analysis of the black-box complexity of \onemax~\revise{led} to the development of the \ollga~\cite{DoerrDE15}. 

The $\jump_k$ function with parameter $k \in [2..n]$ is then defined as follows.

\begin{align*}
    \jump_k(x) = 
    \begin{cases}
        \om(x) + k, \text{ if } \om(x) \in [0..n - k] \cup \{n\}, \\
        n - \om(x), \text{ if } \om(x) \in [n - k + 1..n - 1].    
    \end{cases}
\end{align*}

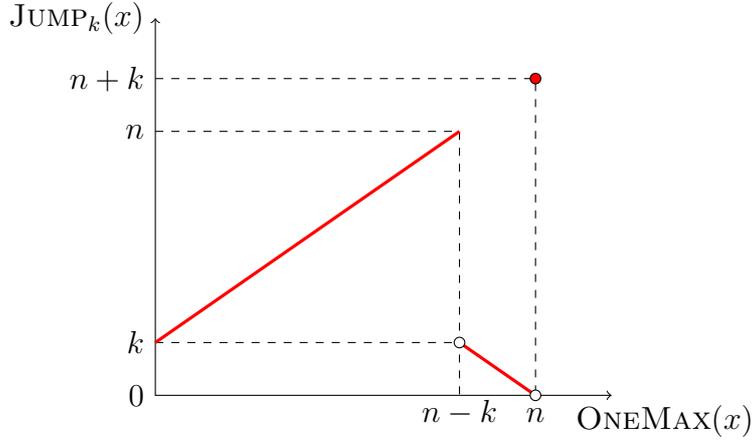
\begin{figure}
    \begin{center}
     \begin{tikzpicture}
  
      \draw [dashed] (0, 4.2) -- (5, 4.2) -- (5, 0);
      \draw [dashed] (0, 3.5) -- (4, 3.5) -- (4, 0);
      \draw [dashed] (0, 0.7) -- (4, 0.7);
  
      \draw [{<[scale=1.5]}-{>[scale=1.5]}] (0, 5) -- (0, 0) -- (6, 0);
      \draw [very thick, red] (0, 0.7) -- (4, 3.5);
      \draw [very thick, red] (4, 0.7) -- (5, 0);

      \draw [fill=white] (4, 0.7) circle (0.7mm);
      \draw [fill=white] (5, 0) circle (0.7mm);
      \draw [fill=red, draw=black] (5, 4.2) circle (0.7mm);
  
      \node [above] at (6.7, -0.7) {$\onemax(x)$};
      \node [above] at (5, -0.5) {$n$};
      \node [left] at (0, 0) {$0$};
      \node [above] at (4, -0.535) {$n - k$};
      \node [left] at (0, 0.7) {$k$};
      \node [left] at (0, 3.5) {$n$};
      \node [left] at (0, 4.2) {$n + k$};
      \node [left] at (0, 5) {$\jump_k(x)$};
  
     \end{tikzpicture}
    \end{center}
    \caption{Plot of the $\jump_k$ function. As a function of unitation, the function value of a search point $x$ depends only on the number of one-bits in~$x$.}
    \label{fig:jump}
  \end{figure}

A plot of $\jump_k$ is shown in Figure~\ref{fig:jump}. 


\subsection{Useful Tools}
In this section we provide some useful tools which we use in our proofs. We start with the following inequality which we use for multiple times in our proofs.

\begin{lemma}\label{lem:bernoulli}
	Assume $x \in \left[0, 1\right]$ and $\lambda > 0$. Then
	\begin{align*}
		1 - (1 - x)^\lambda \ge \frac{1}{2}\min\left\{1, \lambda x\right\}.
	\end{align*}
\end{lemma}
\begin{proof}
	By~\cite[Lemma 8]{RoweS14} we have $(1 - x)^\lambda \le \frac{1}{1 + \lambda x}$. Hence,
	\begin{align*}
        1 - \left(1 - x\right)^\lambda &\ge 1 - \frac{1}{1 + \lambda x} = \frac{\lambda x}{1 + \lambda x} \\
                                       &\ge \frac{\lambda x}{2\max \{1, \lambda x\}} = \frac{1}{2}\min\{\lambda x, 1\}. \qedhere
	\end{align*}
\end{proof}

We also make a use of Chernoff bounds (see Theorem 1.10.1 and 10.10.5 in~\cite{Doerr20bookchapter}) to show the concentration of some random variables involved in our analysis. We use the following lemma, which is a particular case of these bounds for the random variables following a binomial distribution.

\begin{lemma}[Chernoff Bounds]\label{lem:chernoff}
    Let $X$ be a random variable following a binomial distribution $\Bin(n, p)$. 
    Then for all $\delta \in (0, 1)$ the probability that $X \ge (1 + \delta)np$ is at most $e^{-\frac{\delta^2 np}{3}}$ and the probability that $X \le (1 - \delta)np$ is at most $e^{-\frac{\delta^2 np}{2}}$. Also for all $\delta \ge 1$ the probability that $X \ge (1 + \delta)np$ is at most $e^{-\frac{\delta np}{3}}$.
\end{lemma}

The following lemma shows the concentration of the number of the bit flips in the mutation phase by the Chernoff bounds.

\begin{lemma}\label{lem:l-bits}
	Let $p \ge \frac{1}{n}$.
    Then the number $\ell$ of the bits flipped by the mutation operator of the \ollga is in $[p n, 2p n]$ with at least constant 
    probability $q_\ell \ge 0.1$, if $n$ is at least some sufficiently large constant.
\end{lemma}

\begin{proof}
    Recall that the number $\ell$ of the flipped bits is chosen according to the binomial distribution $\Bin(n, p)$. 
    We first consider the case when $p$ is small.
    Assume $p n \in \left[1, 9\right]$. Then 
	\begin{align*}
        \Pr\left[\ell \in [p n, 2p n]\right] &\ge \Pr\left[\ell = \lceil p n\rceil\right] = {n \choose \lceil p n\rceil} p^{\lceil p n\rceil} \left(1 - p\right)^{n - \lceil p n\rceil}\\
        &= \frac{n!}{(n - \lceil pn \rceil)! \lceil pn \rceil!} \cdot \frac{(pn)^{\lceil pn \rceil}}{n^{\lceil pn \rceil}} \cdot (1 - p)^{n - \lceil p n\rceil}. \\
    \end{align*}
    Note that 
    \begin{align*}
      \frac{n!}{(n - \lceil pn \rceil)!n^{\lceil pn \rceil}} &\ge \frac{(n - \lceil p n \rceil  + 1)^{\lceil p n \rceil}}{n^{\lceil p n \rceil}} \ge (1 - p)^9 = (1 - o(1)), \\
      (1 - p)^{n - \lceil p n\rceil} &\ge (1 - p)^{n - np} = (1 - p)^{(\frac{1}{p} - 1) np} \ge e^{-np}. \\
    \end{align*}
    To estimate $\frac{(p n)^{\lceil p n \rceil}e^{-np}}{\lceil p n \rceil!}$ we consider function $f(x) = \frac{x^{\lceil x \rceil}e^{-x}}{\lceil x \rceil!}$ in the interval $[1, 9]$ (since we now consider only these values of $pn$). First we note that $f(1) = e^{-1} \ge 0.3$ and further we find its infimum in $(1, 9]$. Note that $f(x)$ is increasing in each interval $(i, i + 1]$, hence we can bound it from below by $f(x) \ge g(x) = \frac{(\lceil x \rceil - 1)^{\lceil x \rceil}e^{-(\lceil x \rceil - 1)}}{\lceil x \rceil!}$. Note also that $g(x)$ is a non-increasing function in $(1, 9]$, thus $f(x) \ge g(9) \ge 0.12$ for all $x \in (1, 9]$. Consequently, for $n$ which is large enough we have
    \begin{align*}
      \Pr\left[\ell \in [p n, 2p n]\right] &\ge (1 - o(1)) \cdot 0.12 \ge 0.1. 
  \end{align*}

    Now we consider the case when $p n \ge 9$.
	Since $\ell$ follows the binomial distribution with parameters $n$ and $p$, we have $E[\ell] = p n$. By the Chernoff bounds we have
	\begin{align*}
		\Pr\left[\ell \ge 2 E\left[\ell\right]\right] \le \exp\left(-\frac{E\left[\ell\right]}{3}\right) = \exp\left(-\frac{p n}{3}\right).
	\end{align*}
	By Theorem 10 in~\cite{Doerr18exceedexp} we have the following bound on the probability that the binomial distribution exceeds its expectation.
	\begin{align*}
		\Pr[\ell \ge E[\ell] = p n] \ge \frac{1}{4}.
    \end{align*}
    Hence,
    \begin{align*}
		\Pr[\ell < p n] \le \frac{3}{4}.
    \end{align*}
	Therefore, by the union bound the probability $q_\ell$ that $\ell \in [p n, 2p n]$ is at least 
	\begin{align*}	
		q_\ell &= \Pr\left[\ell \ge p n \cap \ell \le 2p n\right] \\ 
		&\ge 1 - \Pr\left[\ell < p n\right] - \Pr\left[\ell > 2p n\right] \\
		&\ge 1 - \frac{3}{4} - \exp\left(-\frac{p n}{3}\right).
	\end{align*}
	Since we assume that $p n \ge 9$, we obtain
	\begin{align*}
		\exp\left(-\frac{p n}{3}\right) \le \exp\left(-3\right) \le 0.05
	\end{align*}
	and hence,
	\begin{align*}
		q_\ell \ge 1 - \frac{3}{4} - \exp\left(-\frac{p n}{3}\right) \ge 0.2.
	\end{align*}
    Therefore
    \begin{align*}
        q_\ell \ge 0.1 = \Omega\left(1\right).
    \end{align*}
\end{proof}

We also state a similar lemma for the larger mutation rates (which are of a greater interest when we aim at escaping the local optimum).

\begin{lemma}\label{lem:l-bits-big-gamma}
	Assume $p = \omega(\frac{1}{n})$.
    Then the number $\ell$ of the bits flipped by the mutation operator is in $[p n, \frac{5}{4}p n]$ with probability $q_\ell' = \frac{1}{4} - o(1)$.
\end{lemma}
\begin{proof}
    By the Chernoff bounds and by Theorem 10 in~\cite{Doerr18exceedexp} we have
    \begin{align*}
        q_\ell' &\ge 1 - \Pr\left[\ell < p n\right] - \Pr\left[\ell > \tfrac{5}{4}p n\right] \\
        &\ge 1 - \tfrac{3}{4} - e^{-\frac{p n}{48}} = 1 - \tfrac{3}{4} - e^{-\omega(1)} = \tfrac{1}{4} - o(1). \qedhere
    \end{align*}
\end{proof}

We also encounter random variables with hypergeometric distribution. A particular example of such random variable is the number $\ell_0$ of zero-bits which are flipped by the mutation operator after the total number $\ell$ of the bits to flip is already chosen. This random variable follows a hypergeometric distribution with parameters $n$, $n - \om(x)$ and $\ell$. For this random variable the Chernoff bounds are also applicable~\cite[Theorem 1.10.25]{Doerr20bookchapter}. We use the following special case of this bound.

\begin{lemma}
    \label{lem:hypergeometric}
    Let $x$ be a bit string with exactly $d$ zero-bits. Let $\ell_0$ be the number of zero-bits of $x$ which are flipped by the mutation operator of the \ollga after $\ell$ is chosen. Then the probability that $\ell_0 > (1 + \delta)\frac{\ell d}{n}$ is at most $\exp(-\frac{\delta^2 \ell d}{3n})$. 
\end{lemma}

\section{Upper Bounds}

In this section we analyse the \ollga with general parameters on $\jump_k$ and show upper bounds on its runtime. We recede from the standard parameter setting of the \ollga, since the intuition behind these parameters values (that is, the intent to have only a single bit flipped if we consequently apply mutation and crossover operators) suggests that they are not efficient to escape local optima.

We observe that the working principles of the \ollga, which were shown to be efficient on \onemax, are also successful on \jump. That is, we show that in the mutation phase the algorithm detects a beneficial mutation (which flips all missing zero-bits of the current individual to ones) by inspecting the fitness of the offspring and then crossover is capable to repair all wrong bit flips made in the mutation phase. We also show that the optimal mutation rate for performing a jump is $\sqrt{\frac{k}{n}}$, which is very different from the optimal one for optimizing \onemax ($\sqrt{\frac{1}{nd}}$ when we are in distance $d$ from the optimum). At the same time the optimal crossover bias, which is also $\sqrt{\frac{k}{n}}$ is very similar to $\sqrt{\frac{d}{n}}$, which is optimal for \onemax.

We split our analysis into two parts. First we find the expected time the \ollga needs to perform a jump to the global optimum when it is already in the local optimum. Then we complete the story by considering the runtime until the \ollga gets to the local optimum starting in a random bit string.

We do not consider the case when $k = 1$, since $\jump_1$ coincides with \onemax, which is already well-studied in the context of the \ollga (see~\cite{DoerrD18} for the full picture). We also omit considering too large values of $k$ (namely, $k > \frac{n}{4}$) since they do not give much new insight about the \ollga, while they require more complicated arguments for our results to hold.

We also constrain ourselves to the case $p \ge \frac{2k}{n}$ so that once we get to the local optimum we have a decent probability to flip at least $2k$ bits. \revise{Since all mutation offspring have the same Hamming distance from the parent,} this implies that an individual with $k$ zero-bits flipped will have a better fitness than any other offspring and therefore selected as the winner of the mutation phase $x'$. Without this assumption an individual with all zero-bits flipped to one might occur in the fitness valley, thus it is not detected as the mutation phase winner. Hence, the jump to the global optimum becomes more challenging for the algorithm, which makes this parameter setting not really promising to be effective on multimodal functions.

\subsection{Escaping the Local Optimum}
\label{sec:jump-to-global}

In this section we analyse how the \ollga leaves the local optimum. Although by runtime we understand the time until the optimum is sampled, it is fair to consider the time until $x$ becomes the optimum for at least two reasons. (i) By disregarding the event that the optimum is sampled in the mutation phase, we still get an upper bound on the runtime. (ii) Since the probability to sample the optimum in the mutation phase is small compared to the probability to sample the optimum in the crossover phase, we expect to lose only a little. Due to the elitist selection the only chance to leave the local optimum is to find the global optimum in one iteration. For this it is sufficient that the following two consecutive events happen.
\begin{enumerate}
    \item The mutation phase winner $x'$ has all $k$ bits which are zero in the current individual $x$ flipped to one.
    \item The crossover winner $y$ takes all $k$ bits which are zero in $x$ from $x'$ and all bits which are zero in $x'$ from $x$.
\end{enumerate}
We first estimate the probability of the first event and then estimate the probability of the second event conditional on the first one.

We call the mutation phase \emph{successful} if all $k$ zero-bits of $x$ are flipped to one in $x'$ (and possibly some one-bits are flipped to zero)
and the number $\ell$ of the flipped bits is at most $2p n$. We estimate the probability $p_M$ of having a successful mutation phase in the following lemma.

\begin{lemma}
    \label{lem:pr-successful-mut}
	Let $k \le \frac{n}{4}$. If $p \ge \frac{2k}{n}$, then we have
	\begin{align*}
		p_M \ge \frac{q_\ell}{2}\min\left\{1, \lambda_M \left(\frac{p}{2}\right)^k\right\},
	\end{align*}
    where $q_\ell$ is as defined in Lemma~\ref{lem:l-bits}, which is $\Theta(1)$.
\end{lemma}

\begin{proof}
    If $\ell \ge 2k$ then we flip at least $k$ one-bits in each mutant, hence the fitness of each mutant is at most $n - k$. Therefore, if there is at least one individual with all $k$ zero-bits flipped,  then this individual has a greater value of $\jump_k$ than any other individual which does not have all zero-bits flipped. Hence, such an individual is chosen as the mutation winner $x'$.
    Therefore, for a successful mutation phase it suffices that the following two events occur (in this order).
	\begin{itemize}
		\item The number of flipped bits $\ell$ is in $[p n, 2p n]$.
		\item The $k$ zero-bits of $x$ are among the $\ell$ chosen bits in at least one of the $\lambda_M$ offspring. We call such offspring \emph{good} in this proof.
	\end{itemize}
	
    By Lemma~\ref{lem:l-bits} the probability of the first event is $q_\ell \ge 0.1$. We condition on this event in the remainder. The probability $q_M(\ell)$ that one particular offspring is good is $\frac{\binom{n - k}{\ell - k}}{\binom{n}{\ell}}$.
	By the assumption that $p \ge \frac{2k}{n}$ we have
	\begin{align*}	
        q_M(\ell) &= \frac{\binom{n - k}{\ell - k}}{\binom{n}{\ell}} = \frac{(n-k)!}{(\ell-k)!(n-\ell)!}\cdot\frac{\ell!(n-\ell)!}{n!} \\
                           &= \frac{\ell(\ell-1)\dots(\ell-k+1)}{n(n-1)\dots(n-k+1)} 
                           \ge \left(\frac{\ell-k}{n}\right)^k \ge \left(\frac{p n}{2n}\right)^k = \left(\frac{p}{2}\right)^k.
	\end{align*}

	 The probability that at least one offspring is good is  $1 - (1 - q_M(\ell))^{\lambda_M}$.
	By Lemma~\ref{lem:bernoulli}, we estimate
	\begin{align*}
        1 - \left(1 - q_M(\ell)\right)^{\lambda_M} &\ge \frac{1}{2}\min\left\{1, \lambda_M q_M(\ell)\right\} \\
        &\ge \frac{1}{2}\min\left\{1, \lambda_M\left(\frac{p}{2}\right)^k\right\}.
	\end{align*}
	Therefore, we conclude
	\begin{align*}
        p_M &\ge \Pr\left[\ell \in \left[p n, 2p n\right]\right] \cdot \frac{1}{2}\min\left\{1, \lambda_M\left(\frac{p}{2}\right)^k\right\} \\
            &\ge \frac{q_\ell}{2}\min\left\{1, \lambda_M\left(\frac{p}{2}\right)^k\right\}. \qedhere
	\end{align*}
\end{proof}


Now we proceed with the crossover phase.
We call the crossover phase \emph{successful} (conditional on a successful mutation phase) if the winner $y$ takes all bits which are zero in $x'$ from $x$ (where they are ones) and all $k$ bits which are zero in $x$ from $x'$ (where they are ones). 
We denote the probability of a successful crossover phase by $p_C$. 
\begin{lemma}\label{lem:pr-successful-cr}
    Assume that $k \le \frac{n}{4}$ and the mutation phase was successful.
    Then
	\begin{align*}
        p_C \ge \frac{1}{2}\min\left\{1, \lambda_C c^k \left(1 - c\right)^{2p n - k}\right\}.
	\end{align*}
\end{lemma}
\begin{proof}
    To generate an optimal solution in one application of the crossover operator we need to take $k$ particular bits from $x'$ and $\ell - k$ particular bits from $x$.
    The probability $q_C$ to generate such a crossover offspring is
	\begin{align*}
        q_C = c^k \left(1 - c\right)^{\ell - k} \ge c^k \left(1 - c\right)^{2p n - k},
	\end{align*}
	since a successful mutation implies that $\ell \le 2p n$. The probability to generate at least one such offspring is
	\begin{align*}
        p_C &= 1 - \left(1 - q_C\right)^{\lambda_C} \ge 1 - \left(1 - c^k \left(1 - c\right)^{2p n - k}\right)^{\lambda_C} \\
            &\ge \frac{1}{2}\min\left\{1, \lambda_C c^k \left(1 - c\right)^{2p n - k}\right\},
	\end{align*}
	where the last inequality follows from Lemma~\ref{lem:bernoulli}.
\end{proof}


With Lemmas~\ref{lem:pr-successful-mut} and~\ref{lem:pr-successful-cr} we are capable of proving the upper bounds on the expected runtime until the \ollga escapes the local optimum. We estimate the runtime both in terms of the number fitness evaluations and the number of iterations, denoted by $T_F$ and $T_I$ respectively. 

\begin{theorem}\label{thm:runtime}
    Let $k \le \frac{n}{4}$. Assume that $p \ge \frac{2k}{n}$
    and $q_\ell$ is as defined in Lemma~\ref{lem:l-bits}.
    Then the expected runtime of \ollga on $\jump_k$ is 
	\begin{align*}
        E\left[T_I\right] \le \frac{4}{q_\ell \min\left\{1, \lambda_M \left(\frac{p}{2}\right)^k\right\} \min\left\{1, \lambda_C c^k \left(1 - c\right)^{2p n - k}\right\}}
    \end{align*}
	iterations and 
	\begin{align*}
        E\left[T_F\right] \le \frac{4(\lambda_M + \lambda_C)}{q_\ell \min\left\{1, \lambda_M \left(\frac{p}{2}\right)^k\right\} \min\left\{1, \lambda_C c^k \left(1 - c\right)^{2p n - k}\right\}}
    \end{align*}
    fitness evaluations if the algorithm starts in the local optimum.
\end{theorem}

\begin{proof}
    When the algorithm is in the local optimum it stays there until it moves to the optimum. During this time in each iteration it has the same probability $P$ to move into the global optimum, which is the probability that a successful mutation phase is followed by a successful crossover phase:
	\begin{align*}
        P = p_M p_C \ge \frac{q_\ell}{2}\lambda_M\left(\frac{p}{2}\right)^k \cdot \frac{1}{2}\lambda_C c^k \left(1 - c\right)^{2p n - k}.
	\end{align*}
	Hence we obtain an expected optimization time in terms of iterations of
	\begin{align*}
        E\left[T_I\right] = \frac{1}{P} \le \frac{4}{q_\ell \lambda_M\lambda_C \left(\frac{p c}{2}\right)^k \left(1 - c\right)^{2p n - k}}.
	\end{align*}
    In each iteration the \ollga performs exactly $\lambda_M + \lambda_C$ fitness evaluations, which gives us an expected number of
	\begin{align*}
        E\left[T_F\right] \le \frac{4 \left(\lambda_M + \lambda_C\right)}{q_\ell \lambda_M \lambda_C \left(\frac{p c}{2}\right)^k (1 - c)^{2p n - k}}
    \end{align*}
    fitness evaluations in total.
\end{proof}


With help of Theorem~\ref{thm:runtime} we deliver good values for the parameters, namely $p = c = \sqrt{\frac{k}{n}}$ and $\lambda_C = \lambda_M = \sqrt{\frac{n}{k}}^k$. We omit the proof that these parameters yield the lowest upper bound (apart from optimizing the $e^{O(k)}$ factor), since it is just a routine work with complicated derivatives, but we state the runtime bounds resulting from these settings in the following corollary. In order to use this result in Section~\ref{sec:total-runtime} we also formulate this theorem for general population sizes.

\begin{corollary}\label{thm:optimal-jump-runtime}
    Let $k \in [2..\lfloor\frac{n}{4}\rfloor]$.
    Assume that $p = c = \sqrt{\frac{k}{n}}$ and
    $\lambda_M = \lambda_C = \lambda \le 2^k\sqrt{\frac{n}{k}}^k$.
    Then the expected runtime of \ollga on $\jump_k$ is  $E[T_F] \le n^k k^{-k} e^{O(k)}\lambda^{-1}$
    fitness evaluations and
        $E[T_I] \le n^k k^{-k} e^{O(k)}\lambda^{-2}$
	iterations, if it starts in a local optimum of $\jump_k$. For $\lambda = \sqrt{\frac{n}{k}}^k$ these bounds are $E[T_I] \le e^{O(k)}$ and $E[T_F] \le \sqrt{\frac{n}{k}}^k e^{O(k)}$.
\end{corollary}

\begin{proof}
    With $\lambda \le 2^k\sqrt{\frac{n}{k}}^k$ we have $\lambda (\frac{p}{2})^k \le 1$ and $\lambda c^k(1 - c)^{2p n - k} \le 1$. Consequently, by Theorem~\ref{thm:runtime} we have
	\begin{align*}
        E\left[T_I\right] &\le  \frac{4}{q_\ell \lambda^2 \left(\frac{k}{2n}\right)^k \left(1 - \sqrt{\frac{k}{n}}\right)^{2 \sqrt{kn} - k}} \\
                          &= \frac{2^{k + 2}\left(\frac{n}{k}\right)^k}{q_\ell \lambda^2 \left(1 - \sqrt{\frac{k}{n}}\right)^{2 \sqrt{kn} - k}} \\
                          &\le \frac{2^{k + 2}}{q_\ell \lambda^2} \left(\frac{n}{k}\right)^k \left(1 - \sqrt{\frac{k}{n}}\right)^{-2 \sqrt{kn}}  \\
                          &\le \frac{2^{k + 2}}{q_\ell \lambda^2} \left(\frac{n}{k}\right)^k \left(1 - \sqrt{\frac{k}{n}}\right)^{-\sqrt{\frac{n}{k}} 2k},
	\end{align*}
    Where $q_\ell \in [0.1, 1]$ is a constant defined in Lemma~\ref{lem:l-bits}. By the estimate $(1 - x)^{-\frac{1}{x}} \le 4$ which holds for all $x \in (0, \frac{1}{2}]$ and by $\sqrt{\frac{k}{n}} \le \sqrt{\frac{n}{4n}} = \frac{1}{2}$ we have
    \begin{align}
        \label{eq:t-i}
        \begin{split}
        E\left[T_I\right] &\le \frac{2^{k + 2}}{q_\ell \lambda^2} \left(\frac{n}{k}\right)^k\left(1 - \sqrt{\frac{k}{n}}\right)^{-\sqrt{\frac{n}{k}} 2k} \\
                          &\le \frac{2^{k + 2}}{q_\ell \lambda^2} \left(\frac{n}{k}\right)^k 4^{2k} \\
                          &= \left(\frac{n}{k}\right)^k \frac{e^{O(k)}}{\lambda^2} .
        \end{split}
    \end{align}
    The expected number of fitness evaluations is $\lambda_M + \lambda_C = 2\lambda$ times greater, hence we have
    \begin{align}
        \label{eq:t-f}
        E\left[T_F\right] \le \left(\frac{n}{k}\right)^k \frac{e^{O(k)}}{\lambda}.
    \end{align}

    Putting $\lambda = \sqrt{\frac{n}{k}}^k$ into~\eqref{eq:t-i} and~\eqref{eq:t-f} we have $E[T_I] \le e^{O(k)}$ and $E[T_F] \le \sqrt{\frac{n}{k}}^k e^{O(k)}$.
\end{proof}

In the following corollary we show a wide range of the parameters, which yield a better upper bound than the mutation-based algorithm (apart from the $e^{O(k)}$ factor) for the sub-linear jump sizes. We do not show it for $k = \Theta(1)$, since in this case the upper bound given by Corollary~\ref{thm:optimal-jump-runtime} is $e^{O(k)}$, which is not better than the runtime of best mutation-based EAs. 

\begin{corollary}
    \label{thm:parameter-range}
    Let $k \ge 2$ and $k = o(n)$.
    Assume that $p = \omega(\frac{k}{n})$, $c = \omega(\frac{k}{n})$ and $p c = O(\frac{k}{n})$.
    Define $\alpha \coloneqq \lambda_M (\frac{p}{2})^k$ and $\beta \coloneqq \lambda_C c^k (1 - c)^{2p n - k}$.
    If $\alpha$ and $\beta$ are at most one and $\alpha = \omega((\frac{k}{nc})^k)$ and $\beta = \omega((\frac{2k}{p n})^k)$, then the expected number of fitness evaluations until the \ollga reaches the global optimum starting from the local optimum of $\jump_k$ is
    \[
      E[T_F] = o\left(\left(\frac{n}{k}\right)^k\right)e^{O(k)}.
    \] 
\end{corollary}

Before we prove the corollary we shortly discuss how one can choose the parameters that give us $o\left(\left(\frac{n}{k}\right)^k\right)e^{O(k)}$ runtime with Corollary~\ref{thm:parameter-range}. First we should choose $p$. It can be any value which is $\omega(\frac{k}{n})$ and which is $o(1)$. Then with the chosen value of $p$ we can choose any $c$ which is on the one hand $\omega(\frac{k}{n})$, but on the other hand $O(\frac{k}{n}p^{-1})$. Note that the closer $p$ is to $\Theta(1)$, the smaller the range for $c$ (thus we could not choose $p = \Theta(1)$, since in this case we cannot simultaneously satisfy $c = \omega(\frac{k}{n})$ and $p c = O(\frac{k}{n})$). 

After we determine $p$ and $c$, we can choose $\lambda_M$ and $\lambda_C$. For $\lambda_M$ the upper bound for the possible range is $(\frac{p}{2})^{-k}$, which follows from condition $\alpha = \lambda_M (\frac{p}{2})^k \le 1$. The lower bound for $\lambda_M$ is $\omega((\frac{2k}{p c n})^k)$, which follows from the condition $\alpha = \omega((\frac{k}{nc})^k)$. For $\lambda_C$ we have similarly obtained bounds, which are $c^{-k}(1 - c)^{-(2p n - k)}$ and $\omega((\frac{2k}{p c n})^k) (1 - c)^{-(2p n - k)}$.

Generally, the choice of the $\lambda_M$ and $\lambda_C$ should be made in such way that they were as close as possible to the inverse probabilities of creating a good offsprings in the mutation and crossover phases respectively. By Lemma~\ref{lem:bernoulli} this choice yields a $\Theta(1)$ probability of a successful iteration. Any smaller population size reduces this probability (usually greater than it reduces the cost of one iteration), while any greater population size only increases the cost of each iteration without significantly increasing the success probability.

\begin{proof}[Proof of Corollary~\ref{thm:parameter-range}]
    Since $\alpha$ and $\beta$ are at most one, the runtime given by Theorem~\ref{thm:runtime} is simplified to
    \begin{align*}
        E[T_F] &\le \frac{4(\lambda_M + \lambda_C)}{q_\ell \lambda_M \lambda_C \left(\frac{p c}{2}\right)^k (1 - c)^{2p n - k}} \\
               &= \frac{4}{q_\ell \alpha c^k (1 - c)^{2p n - k}} + \frac{4}{q_\ell \beta \left(\frac{p}{2}\right)^k}.
    \end{align*}

    We want both terms to be $o((\frac{n}{k})^k)e^{O(k)}$. For the first term it is sufficient if the three following conditions hold. (i) $c^k = \omega((\frac{k}{n})^k)$, which holds if $c = \omega(\frac{k}{n})$. (ii) $(1 - c)^{2p n - k} = e^{-O(k)}$. For this it is sufficient to have $p c = O(\frac{k}{n})$, since then we have
    \begin{align*}
        (1 - c)^{2p n - k} = (1 - c)^{\frac{1}{c}(2p c n - k c)} \ge \left(\frac{1}{4}\right)^{2 p c n} = \left(\frac{1}{4}\right)^{O(k)} = e^{-O(k)}. 
    \end{align*}
    (iii) $\alpha = \omega((\frac{k}{nc})^k)$, since this implies 
    \begin{align*}
        \frac{4}{q_\ell \alpha c^k (1 - c)^{2p n - k}} = \frac{o\left(\left(\frac{nc}{k}\right)^k\right)}{c^k e^{-O(k)}} = o\left(\left(\frac{n}{k}\right)^k\right)e^{O(k)}.
    \end{align*}

    For the second term it is enough that the following two conditions hold. (i) $(\frac{p}{2})^k = \omega((\frac{k}{n})^k)$, for which it is sufficient to have $p = \omega(\frac{k}{n})$. (ii) $\beta$ should not be too small, namely, $\beta = \omega((\frac{2k}{p n})^k)$, since it implies
    \begin{align*}
        \frac{4}{q_\ell \beta \left( \frac{p}{2} \right)^k} = o\left( \left( \frac{p n}{2k} \right)^k \right) \left( \frac{p}{2} \right)^{-k}  = o\left(\left(\frac{n}{k}\right)^k\right). & \qedhere
    \end{align*}
\end{proof}

We find it interesting to show that the standard parameter setting does not give us such a good upper bound. Note, however, that our lower bound given in Theorem~\ref{thm:upper-bound-P} allows the actual runtime of the \ollga with standard parameters to be better. 


\begin{theorem}\label{th:standard-parameters-runtime}
    Let $k \in [2..\lfloor\frac{n}{4}\rfloor]$.
    Assume that $p = \frac{\lambda}{n}$, $c = \frac{1}{\lambda}$ and
    $\lambda_M = \lambda_C = \lambda$ for some $\lambda \in [2k..n]$.
    
    If $\lambda \le (2n)^{\frac{k}{k + 1}}$, then the expected runtime of \ollga on $\jump_k$ is 
        $E[T_I] = O(2^kn^k\lambda^{-2})$
	iterations and 
        $E[T_F] = O(2^kn^k\lambda^{-1})$
    fitness evaluations. 

    If $\lambda \ge (2n)^{\frac{k}{k + 1}}$, then the expected runtime of \ollga on $\jump_k$ is 
        $E[T_I] = O(\lambda^{k - 1})$
	iterations and 
        $E[T_F] = O(\lambda^k)$
    fitness evaluations. 

\end{theorem}

\begin{proof}
    Since the standard parameter setting with $\lambda \ge 2k$ satisfies the conditions of Theorem~\ref{thm:runtime}, we obtain
    \begin{align*}
        E[T_I] &\le \frac{4}{q_\ell\min\left\{1, \lambda \left(\frac{\lambda}{2n}\right)^k \right\} \min\left\{1, \lambda \frac{1}{\lambda^k} \left(1 - \frac 1\lambda\right)^{2\lambda - k} \right\}}.
    \end{align*}

    Note that the second minimum in the denominator is always equal to its second argument. To understand the first minimum we consider two cases.

    \textbf{Case 1.} When $\lambda \le (2n)^{\frac{k}{k + 1}}$, the first minimum is equal to its second argument. Therefore, we have 

    \begin{align*}
      E[T_I] &\le \frac{4}{q_\ell \lambda^2 \left(\frac{\lambda}{2n\lambda}\right)^k \left(1 - \frac{1}{\lambda}\right)^{2\lambda - k}} \\
             &\le \frac{4(2n)^k}{q_\ell \lambda^2 \left(1 - \frac{1}{\lambda}\right)^{2\lambda}} = O\left(\frac{(2n)^k}{\lambda^2}\right).
    \end{align*}

    In each iteration the \ollga performs $2\lambda$ fitness evaluations, thus we have
    \begin{align*}
        E[T_F] &= 2\lambda E[T_I] = O\left(\frac{(2n)^k}{\lambda}\right).
    \end{align*}

    \textbf{Case 2.} When $\lambda \ge (2n)^{\frac{k}{k + 1}}$, first minimum is equal to its first argument. Therefore, we have 

    \begin{align*}
      E[T_I] &\le \frac{4}{q_\ell \lambda \frac{1}{\lambda^k} \left(1 - \frac{1}{\lambda}\right)^{2\lambda - k}} \\
             &\le \frac{4\lambda^{k - 1}}{q_\ell \left(1 - \frac{1}{\lambda}\right)^{2\lambda}} = O\left(\lambda^{k - 1}\right).
    \end{align*}

    In each iteration the \ollga performs $2\lambda$ fitness evaluations, thus we have
    \begin{align*}
        E[T_F] &= 2\lambda E[T_I] = O\left(\lambda^k\right).
    \end{align*}

    Note that this upper bound is minimized when $\lambda = (2n)^{\frac{k}{k + 1}}$ (rounded up or down), in this case the expected runtime is
    \begin{align*}
      E[T_F] &= O\left((2n)^{\frac{k^2}{k + 1}}\right) = O\left((2n)^{k - 1 + \frac{1}{k + 1}}\right).
    \end{align*}

\end{proof}

\subsection{Reaching the Local Optimum}
\label{sec:total-runtime}
In the previous section we showed that the \ollga with non-standard parameters setting can find the global optimum of $\jump_k$ much faster than any standard mutation-based algorithms if the algorithms are started in the local optimum. However, the non-standard parameter setting includes an unnaturally large population size, which makes each iteration costly. At the same time, there is no guarantee that we increase the fitness by much in one iteration, which makes us pay with many fitness evaluations before we reach the local optimum. Hence we question how much the runtime with this parameter setting increases when we start at a random bit string. In this section we show that slightly changing the parameters we can obtain the runtime which is only by a $\sqrt{n}$ factor greater than the runtime when we start in the local optimum. The main result of this section is the following theorem.

\begin{theorem}
    \label{thm:optimal-runtime}
    Let $k \le \frac{n}{4}$. If $\lambda_M = \lambda_C = \frac{1}{\sqrt{n}} \sqrt{\frac{n}{k}}^k $ and $p = c = \sqrt{\frac{k}{n}}$, then the expected runtime of the \ollga with any initialization on the $\jump_k$ function is at most $\sqrt{n}\sqrt{\frac{n}{k}}^k e^{O(k)}$ fitness evaluations. 
\end{theorem}

To prove Theorem~\ref{thm:optimal-runtime} we first analyse the runtime until the \ollga reaches the local optimum of $\jump_k$.

\begin{theorem}
\label{th:local-opt-runtime}
    Let $\lambda_M = \lambda_C = \lambda \ge  \frac{n}{k}$ and $p = c = \sqrt{\frac{k}{n}}$.
    Then the expected time until the \ollga reaches the local optimum of $\jump_k$ with $k \le \frac{n}{4}$ is at most $E[T_I] = ne^{O(k)}$ iterations.
\end{theorem}

The main challenge in the proof of this theorem is that an offspring close to the optimum in the mutation phase can lie in the fitness valley and thus it is not selected as $x'$.
In this section we call the mutation phase \emph{successful} if the winner has at least one zero-bit of $x$ flipped to one.
If such a bit exists, we call it \emph{critical} and we call a mutation phase offspring which has such a bit and does not lie in the fitness valley \emph{good}. We write $p_M$ to denote the probability of a successful mutation phase.
We call the crossover phase \emph{successful} if the winner of the crossover phase has inherited the critical bit from $x'$ and all other bits are 
not changed compared to $x$ (hence, the name ``critical'', since we need such bit to be taken from $x'$ for a successful iteration). We write $p_C$ to denote the probability of this event.

To prove Theorem~\ref{th:local-opt-runtime} we show two auxiliary lemmas for the mutation and crossover phases respectively.
\begin{lemma}
\label{lem:local-opt-mut}
    Let $k \le \frac{n}{4}$. If $\lambda_M \ge \frac{\sqrt{n}}{k\sqrt{k}}$ and $\ell \in \left[\sqrt{nk}, \frac{5\sqrt{nk}}{4}\right]$,
    then $p_M = \Theta(1)$.
\end{lemma}

\begin{proof}
    We denote the current distance to the global optimum by $d$. In order to create a good mutant we need to flip at least one zero-bit, but we also need not to flip too many zero-bits (namely, not more than $\frac{\ell + d - k}{2}$) not to reach the fitness valley. Let $\ell_0$ be the number of the flipped zero-bits in a fixed mutant. Then the probability $q_M$ that the mutant is good is
    \begin{align*}
      q_M = 1 - \Pr[\ell_0 = 0] - \Pr\left[\ell_0 > \frac{\ell + d - k}{2}\right]
    \end{align*}

    We estimate the probability that we have not flipped a single zero-bit using Lemma~\ref{lem:bernoulli}.

    \begin{align*}
        \Pr[\ell_0 = 0] &= \frac{\binom{\om(x)}{\ell}}{\binom{n}{\ell}} \le \left(\frac{n - k}{n}\right)^\ell \le \left(1 - \frac{k}{n}\right)^{\sqrt{nk}} \\
                        &= 1 -  \left(1 - \left(1 - \frac{k}{n}\right)^{\sqrt{nk}}\right) \le 1 - \frac{1}{2}\min\left\{1, \frac{k\sqrt{k}}{\sqrt{n}} \right\}
    \end{align*}

    To estimate the probability that we end up in the fitness valley, we use a Chernoff bound for the hyper-geometric distribution (Lemma~\ref{lem:hypergeometric}). Note that the expected value of $\ell_0$ is $\frac{\ell d}{n}$. Hence, we have

    \begin{align*}
      \Pr\left[\ell_0 > \frac{\ell + d - k}{2}\right] &= \Pr\left[\ell_0 > \left(1 + \left(\frac{\ell + d - k}{2} \cdot \frac{n}{\ell d} - 1\right) \right) \frac{\ell d}{n} \right] \\
      &\le \exp\left(- \frac{1}{3}\left(\frac{\ell + d - k}{2} \cdot \frac{n}{\ell d} - 1\right)^2 \frac{\ell d}{n}\right).
    \end{align*}
    Considering the argument of the exponential as a function of $d$ and computing its derivative (we omit tedious details), one can see that its value is maximized either at $d = \frac{(\ell - k) n}{n - 2\ell}$, if this value lies in $[k + 1, n]$, or at the bounds of this interval.  
    
    \textbf{For $d = \frac{(\ell - k) n}{n - 2\ell}$ we have}
    \begin{align*}
      \Pr\left[\ell_0 > \frac{\ell + d - k}{2}\right] &\le \exp\left(- \frac{1}{3}\left(\frac{\ell - k + \frac{(\ell - k) n}{n - 2\ell}}{2} \cdot \frac{n}{\ell \frac{(\ell - k) n}{n - 2\ell}} - 1\right)^2 \frac{\ell \frac{(\ell - k) n}{n - 2\ell}}{n}\right) \\
      &= \exp\left(- \frac{1}{3}\left(\frac{\left(1 + \frac{n}{n - 2\ell}\right)n}{2\ell \frac{n}{n - 2\ell}} - 1\right)^2 \frac{\ell(\ell - k)}{n - 2\ell}\right) \\
      &= \exp\left(-\frac{1}{3} \left(\frac{n - \ell}{\ell} - 1\right)^2 \frac{\ell(\ell - k)}{n - 2\ell}\right) \\
      &= \exp\left(-\frac{1}{3} \cdot \frac{(n - 2\ell)(\ell - k)}{\ell}\right) 
    \end{align*} 
    Since $n \ge 4k$ and $\ell \ge \sqrt{nk} \ge 2k$, we have $(\ell - k) \ge \frac{\ell}{2}$. Hence, if $\ell \le \frac{n}{3}$, we have
    \begin{align*}
      \Pr\left[\ell_0 > \frac{\ell + d - k}{2}\right] &\le \exp\left(- \frac{(n - 2\ell) \frac{\ell}{2}}{3\ell}\right) \le \exp\left(- \frac{n - \frac{2n}{3}}{6}\right) \le \exp\left(-\frac{n}{18}\right).
    \end{align*} 
    Otherwise, if $\ell > \frac{n}{3}$, then $d = \frac{(\ell - k) n}{n - 2\ell} \le n$ only if $(n - 2\ell) \ge (\ell - k)$. Therefore,
    \begin{align*}
      \Pr\left[\ell_0 > \frac{\ell + d - k}{2}\right] &\le \exp\left(-\frac{1}{3} \cdot \frac{(\ell - k)^2}{\ell}\right) \le \exp\left(-\frac{(\ell / 2)^2}{3\ell}\right) \\
      &= \exp\left(-\frac{\ell}{12}\right) < \exp\left(-\frac{n}{36}\right).
    \end{align*}
    
    \textbf{For $d = k + 1$ we have}
    \begin{align*}
      \Pr\left[\ell_0 > \frac{\ell + d - k}{2}\right] &\le \exp\left(- \frac{1}{3}\left(\frac{(\ell + 1)n}{2\ell (k + 1)} - 1\right)^2 \frac{\ell (k + 1)}{n}\right) \\
      \text{(since $\tfrac{\ell + 1}{\ell} > 1$ and $\tfrac{n}{2(k + 1)} > 1$)}
      &\le \exp\left(- \frac{1}{3}\left(\frac{n}{2(k + 1)} - 1\right)^2 \frac{\ell (k + 1)}{n}\right) \\
      &= \exp\left(-  \frac{\ell \left(n - 2(k + 1)\right)^2}{12(k + 1)n}\right) \\
      \text{(since $\ell \ge 2k$ and $k \le \frac{n}{4}$)}
      &\le \exp\left(- \frac{2k}{12(k + 1)} \cdot \frac{\left(n - \frac{n}{2} - 2\right)^2}{n}\right) \\
      &\le \exp\left(- \frac{1}{9} \cdot \left(\frac{n}{4} - 2\right)\right) \\
      &\le \exp\left(- \frac{n}{36} + \frac{2}{9}\right),
    \end{align*}
    if $n$ is large enough.

    \textbf{For $d = n$ we have}
    \begin{align*}
      \Pr\left[\ell_0 > \frac{\ell + d - k}{2}\right] &\le \exp\left(- \frac{1}{3}\left(\frac{(\ell + n - k)n}{2\ell n} - 1\right)^2 \ell\right) \\
      &\le \exp\left(- \frac{1}{3}\left(\frac{\ell + n - k - 2\ell}{2\ell} \right)^2 \ell\right) \\
      &= \exp\left(-\frac{(n - k - \ell)^2}{12\ell}\right) \\
      \text{(since $k \le \frac{n}{4}$ and $\ell \le \tfrac{5}{4}\sqrt{kn} \le \tfrac{5}{8}n$)}
      &\le \exp\left(-\frac{(n/8)^2}{12 \cdot (5n/8)}\right) = \exp\left(-\frac{n}{480}\right).
    \end{align*}

    Therefore, we estimate the probability to create a good mutant as 
    \begin{align*}
        q_M &\ge 1 - \Pr[\ell_0 = 0] - \Pr\left[\ell_0 > \frac{\ell + d - k}{2}\right] \\
               &\ge \frac{1}{2}\min\left\{1, \frac{k\sqrt{k}}{\sqrt{n}} \right\} - e^{-\frac{n}{480}} 
               \ge \frac{1}{4}\min\left\{1, \frac{k\sqrt{k}}{\sqrt{n}} \right\},
    \end{align*}
    where the last inequality holds when $n$ is at least some sufficiently large constant.
    If $\frac{k\sqrt{k}}{n} > 1$, this probability is already $\Omega(1)$ and hence $p_M = \Omega(1)$.

    Otherwise, by Lemma~\ref{lem:bernoulli} we compute
    \begin{align*}
        p_M &= 1 - (1 - q_M)^\lambda_M \ge \frac{1}{2}\min\left\{1, \lambda_M q_M\right\}.
    \end{align*}
    Since $\lambda_M \ge \frac{\sqrt{n}}{k\sqrt{k}}$, we have $\lambda_M q_M \ge 1$ and therefore, $p_M = \Omega(1)$.

    Finally, we note that for constant $n$ the probability $q_M$ is still positive, and hence $\Omega(1)$.
\end{proof}

We proceed with a lemma for the crossover phase.

\begin{lemma}
    \label{lem:local-opt-cross}
        Let $k \le \frac{n}{2}$. Assume that $c = \sqrt{\frac{k}{n}}$, $\lambda_C \ge \sqrt{\frac{n}{k}}$ and $\ell \in \left[\sqrt{nk}, \frac{5\sqrt{nk}}{4}\right]$,
        and there is at least one critical bit in $x'$.
        Then $p_C = e^{-O(k)}$.
\end{lemma}

\begin{proof}
    To have a successful crossover offspring it is sufficient to take one critical bit from $x'$ and all other different bits from $x$.
    Thus the probability $q_C$ of generating one superior crossover offspring is
    \begin{align*}
        q_C &= c(1 - c)^{\ell - 1} 
            \ge \sqrt{\frac{k}{n}}\left(1 - \sqrt{\frac{k}{n}}\right)^{\frac{5\sqrt{nk}}{4} - 1} \\
            &\ge \sqrt{\frac{k}{n}}\left(1 - \sqrt{\frac{k}{n}}\right)^{\sqrt{\frac{n}{k}}\frac{5k}{4}} 
            = \sqrt{\frac{k}{n}}e^{-\Theta(k)}.
    \end{align*}
    Since we need only one of the $\lambda_C \ge \sqrt{\frac{n}{k}}$ offspring to be superior, by Lemma~\ref{lem:bernoulli} we have
    \begin{align*}
        p_C &= 1 - (1 - q_C)^{\lambda_C} \ge \frac{1}{2} \min\left\{1, \lambda_C\sqrt{\frac{k}{n}} e^{-\Theta(k)}\right\} = e^{-\Theta(k)}. \qedhere
    \end{align*}
\end{proof}

Now we are in position to prove Theorem~\ref{th:local-opt-runtime}
\begin{proof}[Proof of Theorem~\ref{th:local-opt-runtime}]
   We denote the probability to increase fitness in one iteration by $P$ and we estimate this probability as follows.
   \begin{align*}
       P \ge \Pr\left[\ell \in \left[p n, \frac{5p n}{4}\right]\right] \cdot p_M \cdot p_C.
   \end{align*}
   By Lemmas~\ref{lem:l-bits-big-gamma},~\ref{lem:local-opt-mut}, and~\ref{lem:local-opt-cross} we have
   \begin{align*}
       P \ge \left(\tfrac{1}{4} - o(1)\right) \cdot \Omega(1) \cdot e^{-O(k)} = e^{-O(k)}.
   \end{align*}

   Therefore the expected runtime (in terms of iterations) until the \ollga reaches the local optimum of $\jump_k$ is
   \begin{align*}
       E[T_I] \le \sum_{i = 0}^{n - k} \frac{1}{P} \le ne^{O(k)}. & \qedhere
   \end{align*}
\end{proof}

Finally, we prove the main result of this section, Theorem~\ref{thm:optimal-runtime}.

\begin{proof}[Proof of Theorem~\ref{thm:optimal-runtime}]

    By Theorems~\ref{thm:optimal-jump-runtime} and~\ref{th:local-opt-runtime} the upper bound on the total number of fitness evaluations of the \ollga with random initialization is 
    \begin{align*}
        E[T_F] \le \lambda ne^{O(k)} + \sqrt{\frac{n}{k}}^k\frac{e^{O(k)}}{\lambda}.
    \end{align*}

    With $\lambda = \frac{1}{\sqrt{n}}\sqrt{\frac{n}{k}}^k$, we have
    \begin{align*}
        E[T_F] &\le \frac{1}{\sqrt{n}}\sqrt{\frac{n}{k}}^k ne^{\Theta(k)} + \sqrt{\frac{n}{k}}^k\frac{e^{\Theta(k)}}{\frac{1}{\sqrt{n}}\sqrt{\frac{n}{k}}^k}\\
               &\le \sqrt{n}\sqrt{\frac{n}{k}}^k e^{\Theta(k)} + \sqrt{n}\sqrt{\frac{n}{k}}^k e^{\Theta(k)} 
               = \sqrt{n}\sqrt{\frac{n}{k}}^k e^{\Theta(k)}. \qedhere
    \end{align*}
    
\end{proof}

We note that $\lambda = \frac{1}{\sqrt{n}}\sqrt{\frac{n}{k}}^k$ is the value which minimizes our upper bound apart from the $e^{\Theta(k)}$ factor. We omit the proof of this fact, since it trivially follows from the minimization of a function $f(x) = ax + \frac{b}{x}$ via analysis of its derivative.

\section{Lower Bounds}\label{sec:lower-bounds}

In this section, we show that the parameters we chose in the previous section for the optimization starting in the local optimum are asymptotically optimal (apart from $e^{O(k)}$ factors), that is, that no choice of the parameters $\lambda_M$, $\lambda_C$, $p$, and $c$ leads to an asymptotically better runtime (apart from $e^{O(k)}$ factors). This in particular shows that our analysis in the previous section is tight, that is, that the runtime proven in Corollary~\ref{thm:optimal-jump-runtime} is of the right asymptotic order of magnitude (apart from $e^{O(k)}$ factors).

%
\begin{theorem}\label{thm:upper-bound-P}
  The probability $P$ that the \ollga finds the optimum of $\jump_k$ with $k \le \frac{n}{4}$ in one iteration if the current individual is in the local optimum is at most $\lambda_M (\lambda_C + 1) \left(\frac{k}{2n}\right)^k$.
\end{theorem}


Before we prove Theorem~\ref{thm:upper-bound-P} we introduce and prove the following auxiliary lemma. A similar (but less precise) result can be distilled from the proofs of Theorem 4.1 and Corollary 4.2 in~\cite{DoerrLMN17}, but we give a short proof \revise{which also delivers a more precise estimate}.

\begin{lemma}\label{lem:maximizing-probability}
  For all $n \in \N$, all $k \in [1..n - 1]$ and all $x \in [0, 1]$ we have $x^k (1 - x)^{n - k} \le (\frac{k}{n})^k (1 - \frac{k}{n})^{n - k}$. If $n \ge 2k$, then we also have
  $x^k (1 - x)^{n - k} \le (\frac{k}{2n})^k$.
\end{lemma}
\begin{proof}
  Consider the function $f(x) = x^k (1 - x)^{n - k}$. It is smooth in $[0, 1]$, therefore its maxima are in the boundaries of the interval or in the roots of its derivative. Since $f(0) = f(1) = 0$ and $f(x) > 0$ for all $0 < x < 1$, the boundary values cannot be maxima. We compute the derivative as follows.
  
  \begin{align*}
    f'(x) &= kx^{k - 1}(1 - x)^{n - k} - (n - k)x^k(1 - x)^{n - k - 1} \\
          &= x^{k - 1} (1 - x)^{n - k - 1} (k - nx).
  \end{align*}

  The roots of the derivative are in $x = 0$, $x = 1$ and $x = \frac{k}{n}$. Hence the maximum can be reached only in $x = \frac{k}{n}$, and the value of $f$ at this point is $f(\frac{k}{n}) = (\frac{k}{n})^k (1 - \frac{k}{n})^{n - k}$. If we also have $2k \le n$, then furthermore 

  \begin{align*}
    f(x) &\le f\left(\frac{k}{n}\right) = \left(\frac{k}{n}\right)^k \left(1 - \frac{k}{n}\right)^{n - k} = \left(\frac{k}{n}\right)^k \left(\left(1 - \frac{k}{n}\right)^{\frac{n}{k} - 1}\right)^k \\
    &\le \left(\frac{k}{n}\right)^k 2^{-k} = \left(\frac{k}{2n}\right)^k.
   \end{align*}
\end{proof}

Now we are in position to prove Theorem~\ref{thm:upper-bound-P}.

\begin{proof}[Proof of Theorem~\ref{thm:upper-bound-P}]

  We use the precise expression of the probability $P$ to go from the local to the global optimum in one iteration, which is 
  \begin{align}\label{eq:p_static}
      P = \sum_{\ell = 0}^n p_\ell p_M(\ell) p_C(\ell),
  \end{align}
  where $p_\ell$ is the probability to choose $\ell$ bits to flip, $p_M(\ell)$ is the probability of a successful mutation phase conditional on the chosen $\ell$, and $p_C(\ell)$ is the probability of a successful crossover phase conditional on the chosen $\ell$ and on the mutation phase being successful. In contrast \revise{to} the upper bounds, where we have shown a lower bound on the sum of the terms for $\ell \in [np..2np]$, now we aim at giving the upper bound on the whole sum.

  Since $\ell \sim \Bin(n, p)$, we have $p_\ell = \binom{n}{\ell}p^\ell(1 - p)^{n - \ell}$. The probability of a successful mutation phase depends on the chosen $\ell$. If $\ell < k$, then it is impossible to flip all $k$ zero-bits, hence $p_M(\ell) = 0$. For larger $\ell$ the probability to create a good offspring in a single application of the mutation operator is $q_M(\ell) = \binom{n - k}{\ell - k} / \binom{n}{\ell}$. If $\ell \in [k + 1..2k - 1]$\revise{,} then any good offspring occurs in the fitness valley and \revise{has worse fitness} than any other offspring that is not good. Hence, in order to have a successful mutation phase we need all $\lambda_M$ offspring to be good. Therefore, the probability of a successful mutation phase is $(q_M(\ell))^{\lambda_M}$. For $\ell = k$ and $\ell \ge 2k$ we are guaranteed to choose a good offspring as the winner of the mutation phase if there is at least one. Therefore, the mutation phase is successful with probability $p_M(\ell) = 1 - (1 - q_M(\ell))^{\lambda_M}$.

  If $\ell = k$ and the mutation phase is successful, it implies that the optimum is already found and hence we assume $p_C(k) = 1$. Otherwise, we can create a good offspring in the crossover phase only if $\ell > k$. For this we need to take all $k$ bits which are zero in $x$ from $x'$, and then take all $\ell - k$ one-bits which were flipped from $x$. The probability to do so in the creation of one offspring is $q_C(\ell) = c^k (1 - c)^{\ell - k}$. Since we create $\lambda_C$ offspring and at least one of them must have a fitness better than the fitness of $x$, the probability of a successful crossover phase is $p_C(\ell) = 1 - (1 - c^k(1 - c)^{\ell - k})^{\lambda_C}$.

  Putting these probabilities into~\eqref{eq:p_static} we obtain
  \begin{align}\label{eq:P}
    \begin{split}
      P &= \binom{n}{k} p^k (1 - p)^{n - k} \left(1 - \left(1 - \binom{n}{k}^{-1}\right)^{\lambda_M}\right) \\
      &+ \sum_{\ell = k + 1}^{2k - 1} \binom{n}{\ell} p^\ell (1 - p)^{n - \ell} \left(\frac{\binom{n - k}{\ell - k}}{\binom{n}{\ell}}\right)^{\lambda_M} \left(1 - (1 - c^k(1 - c)^{\ell - k})^{\lambda_C}\right) \\
      &+ \sum_{\ell = 2k}^n \binom{n}{\ell} p^\ell (1 - p)^{n - \ell} \left(1 - \left(1 - \frac{\binom{n - k}{\ell - k}}{\binom{n}{\ell}}\right)^{\lambda_M}\right) \left(1 - (1 - c^k(1 - c)^{\ell - k})^{\lambda_C}\right). \\
    \end{split}
  \end{align}

	
	Using Bernoulli's inequality $(1+x)^\lambda \ge 1+\lambda x$ valid for all $x \ge -1$ and $\lambda \ge 1$, we estimate the two terms of type $1 - (1 - p)^\lambda$ by $1 - (1 - p)^\lambda \le 1 - (1 - \lambda p) = \lambda p$. Note that an equivalent estimate could have been obtained by applying a union bound in the probabilistic setting that gave rise to these two expressions.

  We then note that, trivially, we have 
	\[
	\left(\frac{\binom{n - k}{\ell - k}}{\binom{n}{\ell}}\right)^{\lambda_M} \le \lambda_M \frac{\binom{n - k}{\ell - k}}{\binom{n}{\ell}},
	\]
	which allows to uniformly estimate the terms for $\ell \in [k+1..2k-1]$ and $\ell \ge 2k$. 
	
	With these two estimates, we obtain	

  \begin{align*}
    P &\le \binom{n}{k} p^k (1 - p)^{n - k} \lambda_M \binom{n}{k}^{-1} \\
    &+ \sum_{\ell = k + 1}^n \binom{n}{\ell} p^\ell (1 - p)^{n - \ell} \lambda_M \frac{\binom{n - k}{\ell - k}}{\binom{n}{\ell}} \lambda_C c^k(1 - c)^{\ell - k}. \\
  \end{align*}
  
  First, we consider the first term which corresponds to $\ell = k$.
  By Lemma~\ref{lem:maximizing-probability} we have 
  \begin{align}\label{eq:mutation-finds-opt}
    \binom{n}{k} p^k (1 - p)^{n - k} \lambda_M \binom{n}{k}^{-1} = \lambda_M p^k (1 - p)^{n - k} \le \lambda_M \left(\frac{k}{2n}\right)^k.
  \end{align}

  For the rest of the expression we argue as follows.
  \begin{align}\label{eq:pc}
    \begin{split}
      \sum_{\ell = k + 1}^n &\binom{n}{\ell} p^\ell (1 - p)^{n - \ell} \lambda_M \frac{\binom{n - k}{\ell - k}}{\binom{n}{\ell}} \lambda_C c^k(1 - c)^{\ell - k} \\
      &= \sum_{\ell = k + 1}^n \lambda_M \lambda_C \binom{n - k}{\ell - k} p^\ell (1 - p)^{n - \ell} c^k(1 - c)^{\ell - k} \\
      &= \lambda_M \lambda_C (1 - p)^{n - k} (pc)^k \sum_{\ell = k + 1}^n \binom{n - k}{\ell - k} \left(\frac{p(1 - c)}{1 - p}\right)^{\ell - k} \\
      &= \lambda_M \lambda_C (1 - p)^{n - k} (pc)^k \sum_{i = 1}^{n - k} \binom{n - k}{i} \left(\frac{p(1 - c)}{1 - p}\right)^{i} \\
      &\le \lambda_M \lambda_C (1 - p)^{n - k} (pc)^k \left(1 + \frac{p(1 - c)}{1 - p}\right)^{n - k} \\
      &= \lambda_M \lambda_C (pc)^k (1 - pc)^{n - k}.
    \end{split}
  \end{align}
  Since $pc \in [0, 1]$ and $2k \le n$, by Lemma~\ref{lem:maximizing-probability} this expression is at most $\lambda_M\lambda_C (\frac{k}{2n})^k$. Therefore, we conclude that

  \begin{align*}
    P \le \lambda_M \left(\frac{k}{2n}\right)^k + \lambda_M\lambda_C \left(\frac{k}{2n}\right)^k = \lambda_M (\lambda_C + 1) \left(\frac{k}{2n}\right)^k.
  \end{align*}

\end{proof}

Theorem~\ref{thm:upper-bound-P} lets us show in the following corollary that the parameters stated in Corollary~\ref{thm:optimal-jump-runtime} ($\lambda_M = \lambda_C = (\frac{n}{k})^{k/2}$ and $p = c = \sqrt{\frac{k}{n}}$) which minimize the upper bound give us an optimal runtime (apart from an $e^{\theta(k)}$ factor).

\begin{corollary}
  The expected runtime of the \ollga with any parameters on $\jump_k$ is at least $2(\frac{2n}{k})^{k/2} - 1$, if it starts in the local optimum of $\jump_k$. This is of the same asymptotic order (apart from an $e^{\theta(k)}$ factor) as the upper bound shown in Corollary~\ref{thm:optimal-jump-runtime} for the parameters $\lambda_M = \lambda_C = (\frac{n}{k})^{k/2}$ and $p = c = \sqrt{\frac{k}{n}}$
\end{corollary}
\begin{proof}
  Since the probability $P$ to find the global optimum in one iteration is at most $1$, by Theorem~\ref{thm:upper-bound-P} we have
  \begin{align}\label{eq:lower-bound-runtime}
    E[T_F] \ge \frac{\lambda_M + \lambda_C}{P} \ge \max\left\{\lambda_M + \lambda_C, \frac{\lambda_M + \lambda_C}{\lambda_M(\lambda_C + 1) \left(\frac{k}{2n}\right)^k} \right\}. 
  \end{align}
  Consider the arguments of the maximum as functions of $\lambda_M$ (and fix all other parameters). Then $(\lambda_M + \lambda_C)$ is strictly increasing in $\lambda_M$, while 
  \begin{align*}
    \frac{\lambda_M + \lambda_C}{\lambda_M(\lambda_C + 1) \left(\frac{k}{2n}\right)^k} = \frac{\lambda_C}{(\lambda_C + 1) \left(\frac{k}{2n}\right)^k} + \frac{\lambda_C}{\lambda_M(\lambda_C + 1) \left(\frac{k}{2n}\right)^k}
  \end{align*}
  is strictly decreasing. Therefore, the maximum in~\eqref{eq:lower-bound-runtime} is minimized when both its arguments are equal. This condition is satisfied only when $\lambda_M = (\frac{2n}{k})^k (\lambda_C + 1)^{-1}$, which yields the following lower bound on the runtime.
  \begin{align*}
    E[T_F] \ge \lambda_M + \lambda_C \ge \lambda_C + \frac{\left(\frac{2n}{k}\right)^k}{\lambda_C + 1}. 
  \end{align*}
  Considering this as a function of $\lambda_C$ and studying its derivative one can see that this lower bound is minimized when $\lambda_C = (\frac{2n}{k})^{k/2} - 1$. This implies that the minimal lower bound on the runtime is reached with $\lambda_M = (\frac{2n}{k})^k (\lambda_C + 1)^{-1} = (\frac{2n}{k})^{k/2}$ and is equal to
  \begin{align*}
    E[T_F] &\ge \lambda_M + \lambda_C = \left(\frac{2n}{k}\right)^{k/2} + \left(\frac{2n}{k}\right)^{k/2} - 1 \\ 
           &= 2 \left(\frac{2n}{k}\right)^{k/2} - 1 = \left(\frac{n}{k}\right)^{k/2} e^{\Theta(k)}.
  \end{align*}
  This is of the same asymptotical order (apart from an $e^{\theta(k)}$ factor) as the upper bound $(\frac{n}{k})^{k/2} e^{\Theta(k)}$ for parameters $\lambda_M = \lambda_C = (\frac{n}{k})^{k/2}$ and $p = c = \sqrt{\frac{k}{n}}$ shown in Corollary~\ref{thm:optimal-jump-runtime}.

\end{proof}

In this section, we have shown that we determined in Corollary~\ref{thm:optimal-jump-runtime} an asymptotically optimal parameter setting for leaving the local optimum of jump functions. We leave open the question if we also used the best parameters in Theorem~\ref{thm:optimal-runtime}, that is, when starting with a random solution. This is also an interesting question, but it is harder to answer due to the trade-off between optimizing the \onemax-type part of the optimization process and the part leaving the local optimum. In addition, we believe that the question we did answer, how to optimally leave the local optimum, is more interesting from the application point of view. 
\revise{We could imagine that when solving several similar instances of a difficult optimization problem, just by analyzing the optimization processes, one can obtain a rough estimate on the number of bits that need to be flipped to leave the hardest local optima.}
With this information, the parameters determined in Corollary~\ref{thm:optimal-jump-runtime} would be a reasonable starting point for optimizing the parameters of the algorithm. \revise{In constrast to this}, we doubt that in a practical problem one is able to understand both the structure of the local optima and the easy parts of the fitness landscape sufficiently well that then a trade-off as done in Section~\ref{sec:total-runtime} could reasonably be obtained.

\section{Conclusion}

In this first runtime analysis of the \ollga on a multimodal problem, we observed that this algorithm also has a runtime advantage over classic algorithms on multimodal objective functions, and a much more pronounced one. Whereas the advantage in the previous results on unimodal problems was a gain of a logarithmic factor, we have shown here a runtime that is almost the square root of the runtime of classic algorithms.

For the \ollga to show such a good performance, its parameters have to be chosen differently from what was suggested in previous works, in particular, the mutation rate and crossover bias have to be larger. We developed some general suggestions (at the end of Section~\ref{sec:jump-to-global}) that might ease the future use of this algorithm.

Being the first runtime analysis on a multimodal problem, this work leaves a number of questions unanswered. To highlight one of them, we note that we have not proven a matching lower bound for our runtime result. Lower bounds for algorithms with several parameters can be technically demanding as the corresponding analysis~\cite[Section~5]{DoerrD18} of the \ollga on \onemax shows. Hence such a result, despite desirable and possibly also indicating better upper bounds, is beyond the scope of this paper. 

\section*{Acknowledgements}
This work was supported by a public grant as part of the Investissement d'avenir project, reference ANR-11-LABX-0056-LMH, LabEx LMH and by RFBR and CNRS, project number 20-51-15009.

\newcommand{\etalchar}[1]{$^{#1}$}


}

\end{document}